\documentclass[letterpaper, 10 pt, conference]{ieeeconf}  
\IEEEoverridecommandlockouts
\usepackage{amssymb}
\usepackage{graphicx, tikz, subfig}
\usetikzlibrary{arrows.meta}
\usetikzlibrary{decorations.pathreplacing,angles,quotes}
\usetikzlibrary{shapes}
\usepackage{bm}
\usepackage{hyperref}
\usepackage{mathtools}
\usepackage{cuted} 
\usepackage{nccmath}

\newcommand{\br}[1]{\left(#1\right)}
\newcommand{\sbr}[1]{\left[#1\right]}
\newcommand{\set}[1]{\left\{#1\right\}}

\newcommand{\vect}[1]{\mathbf{#1}}
\newcommand{\mtrx}[1]{\mathbf{#1}}

\def\transpose{^\intercal}

\def\Exp{\text{\normalfont Exp}}

\def\vec{\text{\normalfont vec}}

\newcommand{\argmin}{\arg\min}
\newcommand{\argmax}{\arg\max}

\newcommand{\defas}{\doteq}

\makeatletter
\newcommand{\doublewidetilde}[1]{{%
  \mathpalette\double@widetilde{#1}%
}}
\newcommand{\double@widetilde}[2]{%
  \sbox\z@{$\m@th#1\widetilde{#2}$}%
  \ht\z@=.9\ht\z@
  \widetilde{\box\z@}%
}
\makeatother

\DeclareFontFamily{U}{mathx}{\hyphenchar\font45}
\DeclareFontShape{U}{mathx}{m}{n}{
      <5> <6> <7> <8> <9> <10>
      <10.95> <12> <14.4> <17.28> <20.74> <24.88>
      mathx10
      }{}
\DeclareSymbolFont{mathx}{U}{mathx}{m}{n}
\DeclareFontSubstitution{U}{mathx}{m}{n}
\DeclareMathAccent{\widecheck}{0}{mathx}{"71}

\def\vb{\vect{b}}

\def\ve{\vect{e}}
\def\vf{\vect{f}}

\def\vx{\vect{x}}
\def\vy{\vect{y}}
\def\vz{\vect{z}}

\def\vY{\vect{Y}}

\def\vvarpi{\bm{\varpi}}
\def\vvarphi{\bm{\varphi}}

\def\vmu{\bm{\mu}}

\def\mA{\mtrx{A}}

\def\mT{\mtrx{T}}

\def\midentity{\text{\normalfont{\textbf{I}d}}}
\def\idmtrx{\midentity}

\def\mSigma{\bm{\Sigma}}

\def\mOmega{\bm{\Omega}}

\def\NN{\mathbb{N}}

\def\RR{\mathbb{R}}

\def\SS{\mathbb{S}}

\def\calS{\mathcal{S}}

\def\calN{\mathcal{N}}

\def\calY{\mathcal{Y}}

\def\calH{\mathcal{H}}

\newtheorem{prop}{Proposition}[section]
\newtheorem{rmk}{Remark}[section]

\usepackage{flushend}

\begin{document}

\title{\LARGE\bf
Generalization of Auto-Regressive Hidden Markov Models to Non-Linear Dynamics and Unit Quaternion Observation Space
}
\author{Michele Ginesi$^{1}$ and Paolo Fiorini$^{2}$
\thanks{
    *This project has received funding from the European Research Council (ERC) under the European Union’s Horizon 2020 research and innovation programme, ARS (Autonomous Robotic Surgery) project, grant agreement No. 742671.
}
\thanks{
    $^{1}$ Department of Computer Science, University of Verona, Strada le Grazie 15, Verona, Italy
    {\tt\small michele.ginesi@univr.it}
    }%
\thanks{
    $^{2}$ Department of Computer Science, University of Verona, Strada le Grazie 15, Verona, Italy
    {\tt\small paolo.fiorini@univr.it}
    }%
}

\maketitle

\thispagestyle{empty}
\pagestyle{empty}
\begin{abstract}
    Latent variable models are widely used to perform unsupervised segmentation of time series in different context such as robotics, speech recognition, and economics.
    One of the most widely used latent variable model is the Auto-Regressive Hidden Markov Model (ARHMM), which combines a latent mode governed by a Markov chain dynamics with a linear Auto-Regressive dynamics of the observed state.

    \noindent
    In this work, we propose two generalizations of the ARHMM.
    First, we propose a more general AR dynamics in Cartesian space, described as a linear combination of non-linear basis functions.
    Second, we propose a linear dynamics in unit quaternion space, in order to properly describe orientations.
    These extensions allow to describe more complex dynamics of the observed state.

    Although this extension is proposed for the ARHMM, it can be easily extended to other latent variable models with AR dynamics in the observed space, such as Auto-Regressive Hidden semi-Markov Models.
\end{abstract}

\section{INTRODUCTION}

\emph{Hidden Markov Models} (HMMs) \cite{CDSCB10, Vis11, VMIJ12} are a type of graphical model widely used in speech recognition \cite{JR85, Jel97}, hand-writing recognition \cite{NF86}, natural language modeling \cite{MS99a}, and to segment kinematics in the context of \emph{Minimally Invasive Surgery} \cite{RLVVKYH08, VRLKH09, TEKHV12}.
The model consists of a Markov chain governing the evolution of an \emph{hidden} (or \emph{latent}) \emph{mode}.
At each time $t$, the hidden mode $z_t$ emits an observation $\vy_t$ with a certain probability distribution $p( \mathbf{y} _t | z_t)$.

A well-known generalization of HMMs is the \emph{Auto-Regressive} HMM (ARHMM) \cite{EMJ89, JR85, Rab89, Mur98, ER05, KVNP14, GRW16}, in which the observation $\vy_t$ at time $t$ is given by the hidden mode $z_t$ and the previous observed state $\vy_{t-1}$ via a linear Auto-Regressive dynamic.
This means that the current latent mode $z_t$ does not emit an observation, but instead it describes the (linear) vector field governing the evolution of the observed state.

\noindent
This type of model, and its generalizations, have been successfully used to temporally segment robot kinematics.
For instance, in \cite{KVNP14} a \emph{state-based transition} ARHMM (STARHMM) was developed.
This model adds a conditional dependence between the current observed state and the next hidden mode.
Similarly, in \cite{KDNVP15} this model was improved to learn skills by creating a primitive library.

In the literature, most of the improvements to ARHMM have been implemented in the definition of the latent or observation space (e.g. by adding hidden variables \cite{CP10, KDNVP15}), by adding conditional probabilities relation between modes (e.g. by making the next hidden mode dependent on the current observed state \cite{KVNP14, KDNVP15}), or by increasing the number of previous states taken into account in the definition of the AR dynamics.
That is, most of the generalizations of the ARHMM have focused on the \emph{topology} of the graphical model.

On the other hand, at the best of the authors' knowledge, no generalization of the definition of dynamics of the observed state exists in the literature.
For this reason, in this work, we propose to generalize the ARHMM to allow Non-Linear AR dynamics in Cartesian space and linear dynamics in Unit Quaternion space.
The main advantage of our proposed improvement lies in the fact that it can be straightforwardly adapted to work in combination with the modification of the topology of the latent variable model mentioned above.

The paper is structured as follows.
In Section~\ref{sec:arhmm} we present the theory behind ARHMM, recalling the \emph{Expectation Maximization} (EM) algorithm used to infer the set of model parameters from data.
In Section~\ref{sec:nlarhmm} we present our proposed modification to the ARHMM model and the modification to the EM algorithm to infer the model's parameter.
In Section~\ref{sec:experiments} we propose the experiments validating our proposed approach.
In Section~\ref{sec:conclusion} we present the conclusion of the work and possible future extensions.

\section{AUTO REGRESSIVE HIDDEN MARKOV MODELS}
\label{sec:arhmm}

In this section, we provide a formal definition of the Auto-Regressive Hidden Markov Model.

In details, an ARHMM is a model defined as follows:
\[ \calH \defas \set{ \calS, \calY, \Theta = \set{ \vvarpi, \mT, \set{ \mA_s, \vb_s , \bm{\Sigma}_s }_{s\in\calS} } } . \]
$ \calS = \set{ 1, 2, \ldots, S } $ is the set of hidden modes, and the mode at time $t\in{1,2,\ldots,T}$ is denoted by $z_t$.
$\calY = \RR^d$ is the observation space, and the observed state at time $t$ is denoted by $ \mathbf{y}_t \in \mathcal{Y} $.
$ \Theta $ is the set of model parameters.
Vector $\vvarpi = [\varpi_i]_{i\in\calS} $ defines the initial mode probability such that
\( \varpi_i = \Pr(z_1 = i) , \)
for \( i\in\calS \).
Matrix $\mT = [ \texttt{t}_{i\,j}]_{i\in\calS}^{j\in\calS}$ is the transition probabilities matrix between hidden modes such that
\( \texttt{t}_{i\,j} = \Pr(z_{t+1} = j | z_t = i) . \)
The emissions are modeled by the following Auto-Regressive (AR) dynamics with Gaussian white noise:
\begin{equation}
    \vy_t | z_t = s , \, \vy_{t-1} \sim \calN \br{ \vy_t | \mA_s \vy_{t-1} + \vb_s , \mSigma_s} .
    \label{eq:ar_dyn}
\end{equation}


The graphical representation of the model is given in Figure~\ref{fig:arhmm}.
In this representation, a grey dot represents an observed variable, while a white dot denotes a hidden variable.
An arrow between two nodes $a$ and $b$ means that the probability of $b$ depends on the value of $a$.
For a refresh on graphical model representation, we suggest \cite{Bis06}.

Two main algorithms characterize the ARHMM (and the latent-variable models in general): the \emph{Expectation Maximization} (EM) algorithm and the \emph{Viterbi} algorithm.
The former is used to infer the set of parameters $\Theta$ by maximizing the likelihood of the observed data
\( p (\vY | \Theta). \)
The latter, instead, allows extracting the sequence of latent mode $\hat\vz$ that maximize the joint probability distribution $ p (\vY, \vz | \Theta) $ for a given observed sequence $ \vY. $
In short, the EM algorithm allows to learn the model parameters from the data, while the Viterbi algorithm allows to segment an observed sequence once the model parameters are known.

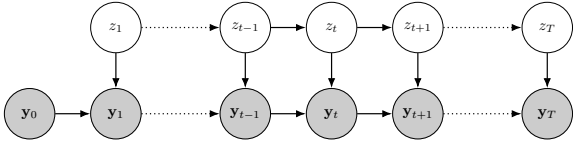
\begin{figure}[t]
    \centering
    \vspace{5pt} %
    \resizebox*{0.9\linewidth}{!}{
        \begin{tikzpicture}
    \node[draw, circle, minimum size = 1.17cm] (z0) at (0,0) {$ z_1 $};
    \node[draw, circle, minimum size = 1.17cm] (z-) at (3,0) {$ z_{t-1} $};
    \node[draw, circle, minimum size = 1.17cm] (z) at (5,0) {$ z_{t} $};
    \node[draw, circle, minimum size = 1.17cm] (z+) at (7,0) {$ z_{t+1} $};
    \node[draw, circle, minimum size = 1.17cm] (zT) at (10,0) {$ z_{T} $};
    \node[draw, circle, minimum size = 1.17cm, fill = black, opacity = 0.2] at (-2,-2) {$ \vy_0 $};
    \node[draw, circle, minimum size = 1.17cm, fill = black, opacity = 0.2] at (0,-2) {$ \vy_1 $};
    \node[draw, circle, minimum size = 1.17cm, fill = black, opacity = 0.2] at (3,-2) {$ \vy_{t-1} $};
    \node[draw, circle, minimum size = 1.17cm, fill = black, opacity = 0.2]at (5,-2) {$ \vy_{t} $};
    \node[draw, circle, minimum size = 1.17cm, fill = black, opacity = 0.2] at (7,-2) {$ \vy_{t+1} $};
    \node[draw, circle, minimum size = 1.17cm, fill = black, opacity = 0.2] at (10,-2) {$ \vy_{T} $};
    \node[draw, circle, minimum size = 1.17cm] (ynot) at (-2,-2) {$ \vy_0 $};
    \node[draw, circle, minimum size = 1.17cm] (y0) at (0,-2) {$ \vy_1 $};
    \node[draw, circle, minimum size = 1.17cm] (y-) at (3,-2) {$ \vy_{t-1} $};
    \node[draw, circle, minimum size = 1.17cm] (y) at (5,-2) {$ \vy_{t} $};
    \node[draw, circle, minimum size = 1.17cm] (y+) at (7,-2) {$ \vy_{t+1} $};
    \node[draw, circle, minimum size = 1.17cm] (yT) at (10,-2) {$ \vy_{T} $};

    \draw[thick, -{Latex}] (ynot) -- (y0);
    \draw[thick, -{Latex}] (z0) -- (y0);
    \draw[thick, -{Latex}] (z-) -- (y-);
    \draw[thick, -{Latex}] (z) -- (y);
    \draw[thick, -{Latex}] (z+) -- (y+);
    \draw[thick, -{Latex}] (zT) -- (yT);
    \draw[thick, -{Latex}, dotted] (z0) -- (z-);
    \draw[thick, -{Latex}] (z-) -- (z);
    \draw[thick, -{Latex}] (z) -- (z+);
    \draw[thick, -{Latex}, dotted] (z+) -- (zT);
    \draw[thick, -{Latex}, dotted] (y0) -- (y-);
    \draw[thick, -{Latex}] (y-) -- (y);
    \draw[thick, -{Latex}] (y) -- (y+);
    \draw[thick, -{Latex}, dotted] (y+) -- (yT);
\end{tikzpicture}
    }
    \caption{Graphical representation of ARHMM.}
    \label{fig:arhmm}
\end{figure}

\subsection{EM Algorithm in brief}

The EM algorithm consists of an iterative procedure that, at each iteration, maximizes, w.r.t. $\Theta$, the auxiliary function
\begin{equation}
    \medmath{
    \begin{multlined}
        Q (\Theta, \Theta^\text{old}) \defas Q_\text{init} (\Theta, \Theta^\text{old}) + Q_\text{trans} (\Theta, \Theta^\text{old}) +
        Q_\text{out} (\Theta, \Theta^\text{old})
    \end{multlined}
    }
    \label{eq:aux_fun}
\end{equation}
where $ \Theta ^ \text{old} $ is the current guess of the set of parameters, and the three terms of the sum are defined as
\begin{subequations}
    \begin{align}
        &
    \medmath{
        \begin{multlined}
            Q_\text{init} (\Theta, \Theta^\text{old}) \defas
            \sum_{z_1\in\calS} \log p(z_1 | \Theta) p (z_1 | \vY , \Theta ^ \text{old}) ,
        \end{multlined}
    }
        \label{eq:aux_init}\\
        &
        \medmath{
        \begin{multlined}
            Q_\text{trans} (\Theta, \Theta^\text{old}) \defas
            \sum_{t=1}^{T-1} \sum_{z_t, z_{t+1} \in \calS} \log p (z_{t+1} | z_t , \Theta)
            p (z_t, z_{t+1}|\vY, \Theta ^ \text{old}) ,
        \end{multlined}
    }
        \label{eq:aux_trans}\\
        &
        \medmath{
        \begin{multlined}
            Q_\text{out} (\Theta, \Theta^\text{old}) \defas
            \sum_{t=0}^{T-1} \sum_{z_{t+1} \in \calS} \log p (\vy_{t+1} | z_{t+1}, \vy_t, \Theta)
            p (z_{t+1} | \vY , \Theta^\text{old}) .
        \end{multlined}
    }
        \label{eq:aux_out}
    \end{align}
\end{subequations}

During the \emph{Expectation Step}, the quantities that depend on $\theta^\text{old}$, that is $p(z_t, z_{t+1} | \vY , \Theta^\text{old})$ and $p(z_{t+1} | \vY , \Theta^\text{old})$, are computed using the \emph{forward-backward algorithm} \cite{Bis06}.
During the \emph{Maximization Step}, the set of parameters $\Theta$ that maximizes $Q_\text{init}$, $Q_\text{trans}$ and $Q_\text{out}$ is computed.
These two steps are repeated until convergence.
Convergence is reached when the changes in the log-likelihood
\( \log p ( \vY , \mathbf{z} ) \)
is below a given tolerance.

\section{GENERALIZATION OF ARHMMS TO DIFFERENT DYNAMICS}
\label{sec:nlarhmm}

In this section, we present our proposed generalizations of the AR dynamic.
In particular, in Section~\ref{subsec:nlarhmm} we discuss our proposed non-linear dynamics in Cartesian space; while in Section~\ref{subsec:uqarhmm} we present our proposed linear dynamics in Unit Quaternion space.

\subsection{Cartesian Non-Linear ARHMM}
\label{subsec:nlarhmm}
When generalizing the AR dynamics in Cartesian space, we aim at modifying the linear AR dynamics shown in \eqref{eq:ar_dyn} with a more general dynamic
\(\vf_s : \RR ^ d \to \RR ^ d\)
for each hidden mode $ s \in \calS .$
To do so, we will use a \emph{basis functions}-based formulation so that the non-linear function $ \vf_s $ is written as a linear combination of non-linear basis functions
\(\set{ \varphi_n:\RR^d \to \RR }_{n=0,1,\ldots, N} : \)
\begin{equation}
    \vf _ s (\vy_t) |_i = f_i^{(s)} (\vy_t) = \sum_{j=0}^N \omega_{i\,j}^{(s)} \varphi_j (\vy_t),
    \label{eq:basis_function_formulation_scalar}
\end{equation}
where $ \omega_{i\,j} ^ {(s)} $ is the set of \emph{weights} that describes the vector field.
By defining the \emph{weight matrix}
\( \bm{\Omega}_s \defas [ \omega_{i\,j}^{(s)} ]_{i = 1,2,\ldots, d} ^ {j = 0,1,\ldots, N}\)
and the \emph{non-linear map}
\( \vvarphi(\vy) \defas [\varphi_j (\vy)]_{j=0,1,\ldots,N} , \)
we have that formulation \eqref{eq:basis_function_formulation_scalar} can be written in matrix-vector notation as
\begin{equation}
    \vf_s(\vy_t) = \mOmega_s \vvarphi(\vy_t) .
    \label{eq:basis_function_formulation_vector}
\end{equation}

\noindent
Thus, \eqref{eq:ar_dyn} now reads
\begin{equation}
    \vy_t | z_t = s , \, \vy_{t-1} \sim \calN \br{ \vy_t | \mOmega_s \vvarphi \br{\vy_{t-1}} , \mSigma_s} .
    \label{eq:nlar_dyn}
\end{equation}

We remark that, since only the formulation of the AR dynamics changes from the classical ARHMM to the non-linear case, the graphical model representation is the same as in Figure~\ref{fig:arhmm}.

During the \emph{Expectation} step, the quantities that depend on the current guess of the set $\Theta^\text{old}$ of parameters have to be computed.
Since the conditional independence properties of the model depend exclusively on its graphical representation \cite{Bis06}, we have that the E-step for the \emph{Non-Linear ARHMM} (NL-ARHMM) is identical to the linear case.
The only difference is that the emission probability
\( p (\vy_t | z_t , \vy_{t-1}) \)
is given by formula \eqref{eq:nlar_dyn} instead of \eqref{eq:ar_dyn}.

Similarly, in the \emph{Maximization} step, maximization of \eqref{eq:aux_init} and \eqref{eq:aux_trans} is achieved in the same way as the linear ARHMM case, since there is no dependence on the emission probability.
Thus, the only difference in the EM algorithm between a linear ARHMM and the NL-ARHMM lies in the maximization of quantity $ Q_\text{out} .$

\begin{prop}
    \label{prop:nl_max}
    Denote by $ \gamma(t) $ the quantity $ p(z_t | \vY, \Theta) $ so that
    \( \gamma_s (t) \defas \Pr( z_t = s | \vY, \Theta ^ \text{\normalfont old} ) .\)
    Then, maximization of $ Q_\text{out} (\Theta, \Theta ^ \text{\normalfont old}) $ in \eqref{eq:aux_out} is achieved by setting
    \begin{subequations}
        \label{eqs:out_max}
        \begin{equation}
        \medmath{
             \mOmega_s =
             \!\! 
             \begin{multlined}[t]
                 \br{ \sum_{t=0}^{T-1} \gamma_s(t+1) \vy_{t+1} \vvarphi (\vy_t)\transpose }
                 \!\! 
                    \br{ \sum_{t=0}^{T-1} \gamma_s(t+1) \vvarphi(\vy_t) \vvarphi(\vy_t)\transpose } ^ {-1}
                \end{multlined}
             \label{eq:out_max_sol_omega}
        }
        \end{equation}
    and
        \begin{equation}
             \mSigma_s =
             \frac{\sum_{t=0}^{T-1} \gamma_s(t+1) \mathbf{e}^{(s)}_t \mathbf{e}^{(s)}_t {} ^ \intercal  }{\sum_{t=0}^{T} \gamma_s (t+1) }
                 ;
             \label{eq:out_max_sol_sigma}
         \end{equation}
    \end{subequations}
     where the vector $ \mathbf{e}_t^{(s)} \in \mathbb{R}^d $ is defined as
     \[ \mathbf{e}_t ^{(s)} = \mathbf{y}_{t+1} - \bm{\Omega}_s \vvarphi ( \mathbf{y}_t). \]
\end{prop}

\begin{proof}
    We aim at maximizing (by dropping $\Theta$ for notation simplicity) the quantity
    \begin{equation*}
        \widetilde{Q}(\Theta, \Theta^\text{old}) = \sum_{t=0}^{T-1} \sum_{s=1}^S \gamma_s(t+1) \log p (\vy_{t+1} | z_{t+1} = s, \vy_t) .
    \end{equation*}
    To maximize this quantity, we can maximize over each hidden mode $ s\in\calS $ independently:
    \begin{equation*}
        \begin{multlined}
                {\mOmega_s^\star, \mSigma_s^\star} =
                    \argmax_{\mOmega_s, \mSigma_s} \sum_{t=0}^{T-1} \gamma_s(t+1)
                    \log p (\vy_{t+1} | z_{t+1} = s, \vy_t).
        \end{multlined}
    \end{equation*}
    Since the emission probability is given by \eqref{eq:nlar_dyn}, the function to maximize reads
    \begin{equation}
        \label{eq:to_max_out_nl}
        \medmath{
            \begin{multlined}[t]
            \widehat{Q}_\text{out} =
            -\frac{1}{2}\sum_{t=0}^{T-1} \gamma_s (t+1) \log |\mSigma_s| \\
            -\frac{1}{2} \sum_{t=0}^{T-1}\gamma_s(t+1) \br{ \vy_{t+1} - \mOmega_s \varphi(\vy_t) }\transpose \mSigma_s^{-1}
                \br{ \vy_{t+1} - \mOmega_s \varphi(\vy_t) }.
        \end{multlined}
    }
    \end{equation}
    We start by discussing the maximization over the weight matrix $\mOmega_s$.
    The gradient of $\widehat{Q}_\text{out}$ w.r.t. $ \bm{\Omega}_s$ is
    \begin{equation*}
        \begin{multlined}
            \nabla_{\mOmega_s} \widehat{Q} = \mSigma_s^{-1} \Bigg(
                \sum_{t=0}^{T-1} \gamma_s(t+1) \vy_{t+1} \vvarphi(\vy)\transpose \hspace{2cm}\\
                -
                \mOmega_s \sum_{t=0}^{T-1} \gamma_s(t+1) \vvarphi(\vy_t) \vvarphi(\vy_t)\transpose
            \Bigg).
        \end{multlined}
    \end{equation*}
    By setting it to zero, we get formula \eqref{eq:out_max_sol_omega}.

    We now discuss the maximization over the weight matrix $\mSigma_s$.
    The gradient of $\widehat{Q}_\text{out}$ w.r.t. $ \bm{\Sigma}_s$ is
    \begin{equation*}
        \nabla_{\mSigma_s ^ {-1}} \widehat{Q} =
        \begin{multlined}[t]
            \frac{1}{2} \mSigma_s \sum_{t=0}^{T-1}\gamma_s(t+1)
            -\frac{1}{2} \sum_{t=0}^{T-1}
            \gamma_s(t+1) \\
            \br{\vy_{t+1} - \mOmega_s \vvarphi(\vy_t)} \br{\vy_{t+1} - \mOmega_s \vvarphi(\vy_t)}\transpose,
        \end{multlined}
    \end{equation*}
    which, by setting it to zero, gives formula \eqref{eq:out_max_sol_sigma}.

\end{proof}

We remark that the matrix $\mOmega_s$ in \eqref{eq:out_max_sol_sigma} is the ``old'' guess, and not that given by the update formula \eqref{eq:out_max_sol_omega}.
Thus, when implementing the EM-algorithm, the update of the covariance matrices $ \mSigma_s $ should be performed before the update of the weight matrices $ \mOmega _ s$.

\subsubsection{Example of Basis Functions}

In this section, we present some examples of basis functions.

The first set shows how to interpret the classical, linear ARHMM as a particular case of NL-ARHMM.
Indeed, by defining
\( \vvarphi^\text{(lin)}(\vy) \defas \sbr{ 1, y_1, y_2, \ldots, y_d } \transpose , \)
we have that the resulting weight matrix is a block matrix with the offset $\vb$ and the linear map $\mA$ in \eqref{eq:ar_dyn} as
\( \mOmega = \sbr{ \vb , \mA } . \)

The second well known set of basis functions is the family of \emph{Gaussian Radial Basis Functions} (GRBFs).
Given a set $\{ \vmu_i \}_{i=1,2,\ldots,N}$ of centers and a set $\{ \bm{\varSigma}_i \}_{i=1,2,\ldots,N}$ of covariance matrices, we define the basis functions as
\begin{align*}
    \varphi^{\text{(grbf)}}_0(\vy) & \defas 1 , \\
    \varphi^{\text{(grbf)}}_i(\vy) & \defas
        \begin{multlined}[t]
            \exp\br{ - (\vy - \vmu_i) \bm{\varSigma}_i^{-1} (\vy - \vmu_i) },
            i=1,\ldots,N.
        \end{multlined}
\end{align*}
Usually, the covariance matrices are set to be a multiple of the identity matrix
\( \bm{\varSigma}_i = \varsigma_i \idmtrx_d . \)
The main drawback of this family of basis functions lies in the fact that, even for small values of the dimension $d$ of the observed space, a high number of basis functions is needed to `cover' the space.
Indeed, this family of basis functions is usually adopted when learning functions in bounded domains (e.g. in \emph{Dynamic Movement Primitives}, where the time domain is fixed \cite{INS03}).
For this reason, we suggest to use this set of basis functions when dealing with bounded observation spaces, such as $d-$dimensional cubes $[0,1]^d$.

Finally, a third family of basis functions is the set of polynomial functions up to a degree $k$:
\begin{equation}
    \vvarphi^{ (\mathcal{P}_k) } (\vy) \defas \text{col}\br{\set{ \prod _ {i = 1} ^ {d} y_i ^ { c_i } : c_i \in \NN,\, \sum_{i=1}^d c_i \le k }} .
    \label{eq:poly_bfs_def}
\end{equation}
For instance, assume that $d=2$ ($\vy \in \mathcal{Y} = \RR^2$) and $k=3$.
Then the basis functions are
\begin{equation*}
    \vvarphi^{(\mathcal{P}_3)}(\vy) = \sbr{
        1,y_1,y_2,y_1^2, y_1y_2, y_2^2,
        y_1^3, y_1^2y_2, y_1y_2^2, y_2^3
    }\transpose.
\end{equation*}
As we will show in Section~\ref{sec:poly_test}, this type of basis functions is particularly useful in low-dimension spaces, where linear dynamics are not able to describe complex evolutions of the observed state.
On the other hand, they are less useful in higher dimensional spaces for two main reasons:
firstly, linear dynamics are able to describe more complex evolutions;
secondly, non-linearity in higher dimensional spaces require a huge amount of parameters (i.e. the number of columns of $\mOmega$ increases) risking over-fitting and numerical inefficiency.

\subsection{Unit-Quaternion Linear ARHMM}
\label{subsec:uqarhmm}

Orientations can be modeled in different ways.
There are two preferred spaces to describe orientations: the space $SO(3)$ of \emph{orthogonal $3\times 3$ matrices}, and the space $\mathbb{S}^3$ of unit quaternions.
Both representations are singularity-free.
However, unit quaternions are preferred since they require only four variable to be described, instead of the nine parameters of $3 \times 3$ matrices.

The main difficulty when dealing with dynamics in $\SS^3$ lies in the fact that the resulting quaternion must still be of unitary norm.
To solve this problem, we take advantage of two properties of quaternions.
The first is the fact that the exponential of a quaternion with null real part always results in a quaternion of unit norm:
\( \left\Vert \Exp(0 + a \texttt{i} + b\texttt{j} + c \texttt{k}) \right\Vert = 1,\quad \forall a,b,c,\in \mathbb{R} . \)
The second is the fact that the norm of the product between two quaternion $p $ and $q$ is the product of the norms:
\( \left\Vert p * q \right\Vert = \left\Vert p \right\Vert\, \left\Vert q \right\Vert .  \)
Thus, we propose to describe a dynamics in $\SS^3$ by 3 parameters $a,b,c\in \mathbb{R} $ such that
\[ q_{t+1} = \Exp (a \texttt{i} + b\texttt{j} + c \texttt{k}) * q_t.  \]
From here, we will denote as `$\vec$' the function that maps a quaternion to its vector formulation:
\begin{equation*}
\medmath{
     \vec : \SS^3 \to \mathbb{R}^4 , \qquad
     q = q_r + q_i\texttt{i} + q_j \texttt{j} + q_k\texttt{k} \mapsto \mathbf{q} = [q_r, q_i, q_j, q_k] ^ \intercal.
}
\end{equation*}
With this notation, we have that in this model the probability of the next observation reads
\begin{multline}
    \label{eq:emission_prob_quat}
    q_t | z_t = s , q_{t-1} \sim \\
    \mathcal{N} \big( \vec(q_t) | \vec\left(\Exp ( a_s \texttt{i} + b_s \texttt{j} + c_s \texttt{k} ) * q_{t-1}\right) , \bm{\Sigma}_s \big)
\end{multline}

As in the case of Cartesian Non-Linear ARHMM, we need only to explain how to maximize the \emph{emission probability}.
Therefore, in the \emph{Maximization step} we have to maximize the quantity $Q_\text{out}$ in \eqref{eq:aux_out} where the next-state probability is given by \eqref{eq:emission_prob_quat}.
Differently from what we showed in Proposition~\ref{prop:nl_max}, there is not a closed form solution to the maximization problem.
Indeed, by defining $ \gamma_s(t) \defas \Pr (z_t = s| \vY , \Theta ^ \text{old}) $ we have that the function to maximize reads, similarly to \eqref{eq:to_max_out_nl}:
\begin{equation}
    \label{eq:q_out_quat}
    \medmath{
    \begin{multlined}[t]
    \widehat{Q}_\text{out} =
        -\frac12 \sum_{t=0}^{T-1} \gamma_s(t+1) \log \left\vert \bm{\Sigma}_s  \right\vert \\
        -\frac12 \sum_{t=0}^{T-1} \gamma_s(t+1) \vec (q_{t+1} - \mu_{t+1}^{(s)}) ^ \intercal \bm{\Sigma} ^ {-1} \vec(q_{t+1} - \mu_{t+1}^{(s)} )
    \end{multlined}
}
\end{equation}
where $\mu_{t+1}^{(s)}$ is the \emph{expected orientation} at time $t+1$ for hidden mode $s$:
\[ \mu_{t+1}^{(s)} \defas \Exp (a_s \texttt{i} + b_s\texttt{j} + c_s\texttt{k}) * q_t. \]
Similarly to Proposition~\ref{prop:nl_max}, maximization of \eqref{eq:q_out_quat} is achieved by setting
\[ \mSigma_s = \frac{\sum_{t=0}^{T-1} \gamma_s(t+1) \ve_{t+1}^{(s)} \cdot \mathbf{e}_{t+1}^{(s)}{} ^ \intercal }{\sum_{t=0}^{T-1} \gamma_s(t+1)} \]
where $\mathbf{e}_{t+1}^{(s)}$ is the difference between the actual orientation and the predicted one at time $t+1$ assuming that the hidden mode $s$ is active:
\[ \mathbf{e}_{t+1}^{(s)} \defas \vec (q_{t+1} - \mu_{t+1}^{(s)}). \]

On the other hand, maximization over the dynamics parameters can be formulated as
\begin{equation*}
\medmath{
    \begin{multlined}
        (a_s^\star, b_s^\star, c_s^\star) = \argmax_{(a_s,b_s,c_s)} \\
        \left(-\sum_{t=0}^{T-1} \gamma_s(t+1) \left\Vert \vec \left( q_{t+1} - \Exp \left( a_s\texttt{i} + b_s \texttt{j} + c_s\texttt{k} \right)* q_t  \right)  \right\Vert_{ \bm{\Sigma} ^ {-1} } ^ 2\right)
    \end{multlined}
}
\end{equation*}
where $ \left\Vert \cdot \right\Vert_ \mathbf{A}  $ is the semi-norm induced by a Symmetric Semi-Positive Definite (SPD) matrix $ \mathbf{A} $:
\( \left\Vert \mathbf{x} \right\Vert_ \mathbf{A} \defas \sqrt{ \mathbf{x} ^ \intercal  \mathbf{A} \mathbf{x} }.  \)
Unlikely the Cartesian formulation, this problem cannot be solved in closed form.
Therefore, to solve the Maximization step, we rewrite the maximization problem as an equivalent minimization problem
\begin{equation*}
    \medmath{
        \begin{multlined}
        (a_s^\star, b_s^\star, c_s^\star) = \argmin_{(a_s,b_s,c_s)} \\
        \left(\sum_{t=0}^{T-1} \gamma_s(t+1) \left\Vert \vec \left( q_{t+1} - \Exp \left( a_s\texttt{i} + b_s \texttt{j} + c_s\texttt{k} \right)* q_t  \right)  \right\Vert_{ \bm{\Sigma} ^ {-1} } ^ 2\right)
        \end{multlined}
    }
\end{equation*}
and solve it via any minimization algorithm (e.g. \emph{gradient descent}).

The implementation (in Python 3.10) of our proposed generalizations is available at \url{https://github.com/mginesi/nl_arhmm}.

\begin{rmk}
    The use of our proposed generalizations can be combined with other extensions to the ARHMM model.
    For instance, in the EM algorithm of \emph{Auto-Regressive Hidden semi Markov Models}, the maximization formula for the emission probabilities can be easily extended in a similar fashion to what we proposed here.
\end{rmk}

\begin{rmk}
    We do not discuss the generalization of the Viterbi algorithm \cite{For73} since its most general formulation works with any latent-mode model with Markov dynamic, independently of the dynamics of the observed state.
\end{rmk}

\section{EXPERIMENTS}
\label{sec:experiments}

In this section, we will compare the results obtained with the NL-ARHMM against the linear ARHMM, showing that our proposed approach results in higher segmentation scores.

\subsection{Validation Test}
\label{sec:val_test}

At first, we show that the EM algorithm for the NL-ARHMM allows learning the parameters of the model.
To do so, we define a NL-ARHMM with fixed parameters and use it to generate data samples.
In our test, we set both the number of hidden modes and the dimension of the continuous state to 2:
\( S = 2, d = 2 . \)
The initial mode distribution is set to
\( \vvarpi = [0.5, 0.5]\transpose \)
and the transition probability matrix is
\[ \mT = \begin{bmatrix}
    0.95 & 0.05 \\
    0.05 & 0.95
\end{bmatrix} . \]
The non-linear dynamics is written as follows.
We start by defining the vector field
\[ \vf
    \left(
        \begin{bmatrix}
            y_1 \\ y_2
        \end{bmatrix}
    \right)
    = \begin{bmatrix}
    y_1{}^3 + y_2{}^2y_1 - y_1 - y_2 \\
    y_2{}^3 + y_1{}^2y_2 + y_1 - y_2
\end{bmatrix} .\]
Then the dynamics of the two hidden modes are obtained using the Euler method with time step $\delta_t = 0.05$ for the dynamics $\vf$ and its opposite $-\vf$, with Gaussian noise with standard deviation $\varsigma = 5\texttt{e}-03$:
\[ %
\begin{aligned}
    \vy_{t+1} | z_{t+1} = 1, \vy_t \sim \calN( \vy_{t+1} | \vy_t + \delta_t \vf(\vy_t) , \varsigma \midentity_2) , \\
    \vy_{t+1} | z_{t+1} = 2, \vy_t \sim \calN( \vy_{t+1} | \vy_t - \delta_t \vf(\vy_t) , \varsigma \midentity_2) . \\
\end{aligned}
\] %
\begin{figure}
    \centering
    \includegraphics[width=0.7\linewidth]{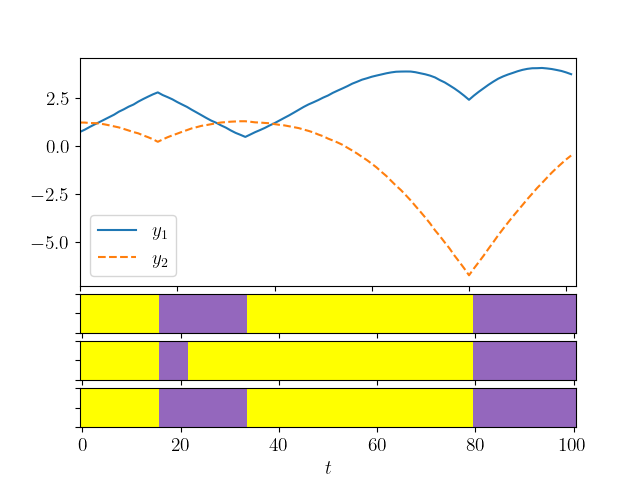}
    \caption{
        Results for the validation test presented in Section~\ref{sec:val_test}.
        The top figure shows the components of the trajectory,
        the second plot shows the ``true'' segmentation (the first mode is represented in purple, while the second mode is represented in yellow),
        the third plot shows the segmentation obtained with the learned ARHMM,
        and the fourth plot shows the segmentation obtained with the learned NL-ARHMM.
        }
        \label{fig:nlarhmm_res}
\end{figure}

\begin{figure*}[t] 
  \centering
  \subfloat[$d=1$,\label{subfig:poly_comp_1}]{
    \includegraphics[width=0.33\linewidth]{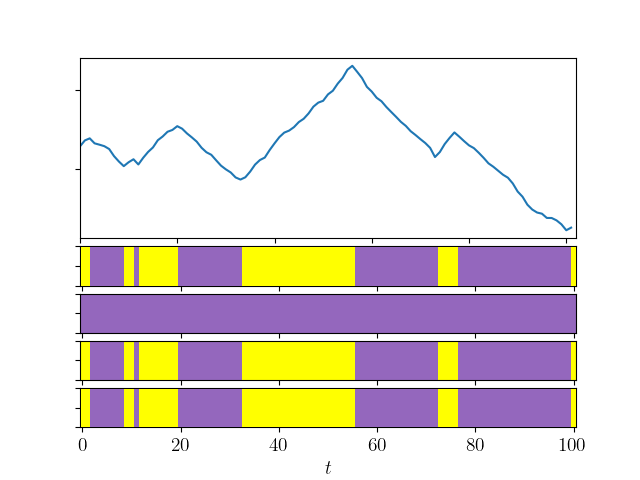}
  }
  \subfloat[$d=2$\label{subfig:poly_comp_2}]{
    \includegraphics[width=0.33\linewidth]{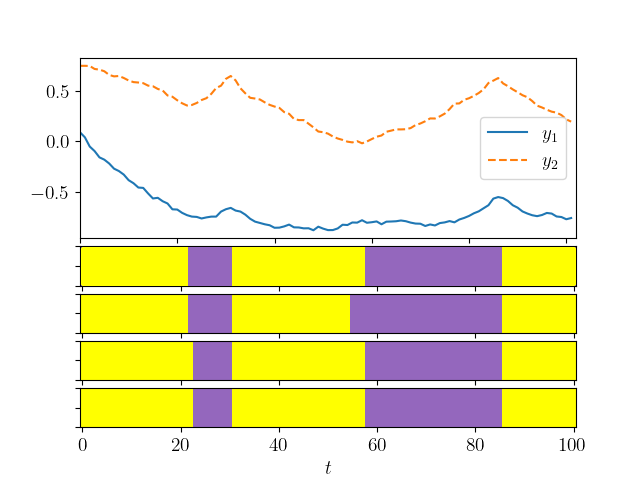}
  }
  \subfloat[$d=3$\label{subfig:poly_comp_3}]{
    \includegraphics[width=0.33\linewidth]{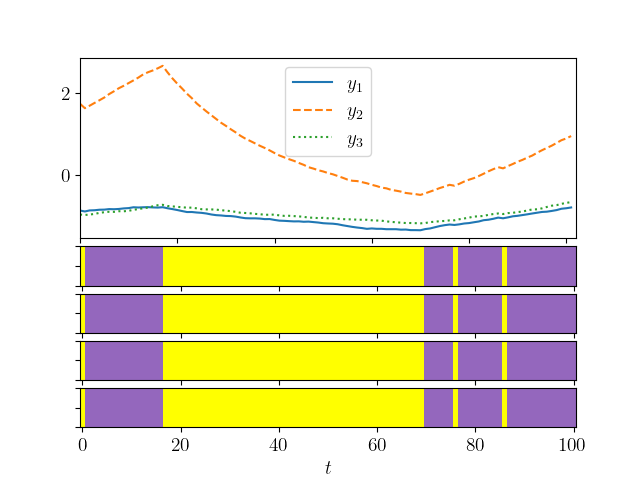}
  }
  \caption{
    Comparison between linear and polynomial ARHMM.
    In all three figures the first plot shows the components $y_i$ of the observed dynamics,
    the second plot shows the true segmentation,
    the third, fourth, and fifth plots shows the segmentations obtained by, respectively, the linear, quadratic, and cubic formulations of the ARHMM (i.e. $k=1,2,3$ in \eqref{eq:poly_bfs_def}).
  }
  \label{fig:poly_comp}
\end{figure*}

\noindent
We use this NL-ARHMM to generate fifty samples with $T = 100$ steps each.
Next, we initialize both a NL-ARHMM (with polynomial basis functions of degree $k=2$) and a linear ARHMM, and apply the EM algorithm to learn the model parameters.
Finally, we use the Viterbi algorithm to segment a new trajectory generated from the ``true'' NL-ARHMM.
Before applying the EM algorithm we standardize the dataset.
That is, we translate it and multiply by a constant to have null mean and unit variance.

\noindent
Figure~\ref{fig:nlarhmm_res} shows one result for these tests.
As can be observed, the NL-ARHMM is able to properly segment the trajectory, being able to capture more complex dynamics.
On the other hand, the linear ARHMM results in a poorer segmentation.

\subsection{Comparison of Different Degrees Polynomial Functions}
\label{sec:poly_test}

In this Section, we propose some tests to show that Non-Linear basis functions are more effective when the observation space $ \mathcal{Y} $ is low dimensional.
On the other hand, for high-dimensional observation space the difference between the two is less noticeable.

To show this aspect we simulate three different ARHMM with non linear dynamics.
Each of them has two hidden modes but differs in the dimension $d$ of the observation space; $d=1,2,3$.
To present a fair comparison, all the six dynamics (two for each value of $d$) are non-polynomial:
in this way none of the model we will compare (linear, quadratic, and cubic) is able to perfectly describe the dynamics.

Simularly to the previous test, for each value of $d$ we generate a total of fifty trajectories of 100 steps.
Figure~\ref{fig:poly_comp} shows the results of these tests.

As it can be observed, with $d=1$ (Figure~\ref{subfig:poly_comp_1}) the linear model is not able to properly describe the non-lienar dynamics, failing in the segmentation of the obtained trajectory.
On the other hand, for $d=2$ and $3$ (Figures~\ref{subfig:poly_comp_2} and \ref{subfig:poly_comp_3} respectively), the differences between the three dynamical model (linear, quadratic, and cubic) significantly reduce.

These tests show that for lower-dimensional observation spaces, the improvements in using the Non-Linear formulation for the ARHMM dynamics are more significative.
We argue that the motivation lies in the fact that for high-dimensional data, linear vector fields can describe a wider range of dynamics, since the number of parameters in the dynamics grows quadratically with the space dimension
(to be precise, the number of parameters for each dynamics is $d^2+d$: $d^2$ elements for the matrix $ \mathbf{A}_s$ and $d$ for the vector $ \mathbf{b}_s $ in \eqref{eq:ar_dyn}).
Oh the other hand, for smaller observation spaces, linear dynamics are too limited and the usage of polynomial functions in the definition of the dynamics provides a broader set of possible behaviors.

For this reason, in Section~\ref{sec:exp_real} we will propose a model in which the Cartesian position of the robot end-effector is modeled by a linear dynamics, while the gripper angle (a 1-dimensional variable) will be modeled using a polynomial dynamics.

\subsection{Experiments on Real Setups}
\label{sec:exp_real}

To prove the effectiveness of our generalizations, we propose a model combining both Non-Linear Cartesian and Linear Unit-Quaternion dynamics.
In particular, we propose a pose+gripper model in which, for each hand of the robot, we model the position with a linear Cartesian model, the end-effector orientation with a linear Unit-Quaternion model, and the gripper angle with a quadratic model.
The decision to use a linear model for the position and a quadratic model for the gripper angle are motivated by the tests presented in Section~\ref{sec:poly_test}: for 3-dimensional spaces, linear, quadratic, and cubic models are almost indistinguishable (Figure~\ref{subfig:poly_comp_3}) and we thus choose the simplest one.
On the other hand, for 1-dimensional trajectories, linear models fail to describe the complexity of the behavior, while a quadratic and a cubic one give similar results (Figure~\ref{subfig:poly_comp_3}).

\noindent
Additionally, we assume that position $ \mathbf{x}^h_t $, orientation $ q^h_t $, and gripper angle $ \vartheta^h_t$ of each arm $h$ are independent of each others when the hidden mode is given:
\begin{multline}
    \label{eq:indep_pose_angle}
    p \left(\substack{ \mathbf{x}^1_{t+1} , q^1_{t+1}, \vartheta ^ 1 _ {t+1} ,\\\\
        \ldots,\\\\
        \mathbf{x}^H_{t+1} , q^H_{t+1}, \vartheta ^ H _ {t+1}}
        \Bigg|
        \substack{\mathbf{x}^1_{t} , q^1_{t}, \vartheta ^ 1 _ {t},\\\\
        \ldots ,\\\\
        \mathbf{x}^H_{t} , q^H_{t}, \vartheta ^ H _ {t} , z_{t+1}}\right)
    = \\
    \prod_{h=1} ^ H p( \mathbf{x}_{t+1}^h | z_{t+1} , \mathbf{x}_t^h) \, p( q_{t+1}^h | z_{t+1} , q_t^h) \, p( \vartheta_{t+1}^h | z_{t+1} , \vartheta_t^h).
\end{multline}
A graphical model representation of the model is given in Figure~\ref{fig:pose_gripper_graph_mod}.

\begin{figure}[t]
    \centering
    \vspace{5pt} %
    \resizebox{1.0\linewidth}{!}{
        \begin{tikzpicture}
    \node[draw, circle, minimum size = 1.17cm] (z-) at (0,0) {$ z_{t-1} $};
    \node[draw, circle, minimum size = 1.17cm] (z) at (5,0) {$ z_{t} $};
    \node[draw, circle, minimum size = 1.17cm] (z+) at (10,0) {$ z_{t+1} $};

    \node[draw, circle, minimum size = 1.17cm, fill = black, fill opacity = 0.2, text opacity = 1](x1-) at (00 + 1.5 + 1.0,-2.0) {$ \vx_{t-1}^1 $};
    \node[draw, circle, minimum size = 1.17cm, fill = black, fill opacity = 0.2, text opacity = 1] (x1) at (05 + 1.5 + 1.0,-2.0) {$ \vx_{t}^1 $};
    \node[draw, circle, minimum size = 1.17cm, fill = black, fill opacity = 0.2, text opacity = 1](x1+) at (10 + 1.5 + 1.0,-2.0) {$ \vx_{t+1}^1 $};
    \node[draw, circle, minimum size = 1.17cm, fill = black, fill opacity = 0.2, text opacity = 1](q1-) at (00 + 1.5 + 0.5,-3.5) {$ q_{t-1}^1 $};
    \node[draw, circle, minimum size = 1.17cm, fill = black, fill opacity = 0.2, text opacity = 1] (q1) at (05 + 1.5 + 0.5,-3.5) {$ q_{t}^1 $};
    \node[draw, circle, minimum size = 1.17cm, fill = black, fill opacity = 0.2, text opacity = 1](q1+) at (10 + 1.5 + 0.5,-3.5) {$ q_{t+1}^1 $};
    \node[draw, circle, minimum size = 1.17cm, fill = black, fill opacity = 0.2, text opacity = 1](t1-) at (00 + 1.5 + 0.0,-5.0) {$ \vartheta_{t-1}^1 $};
    \node[draw, circle, minimum size = 1.17cm, fill = black, fill opacity = 0.2, text opacity = 1] (t1) at (05 + 1.5 + 0.0,-5.0) {$ \vartheta_{t}^1 $};
    \node[draw, circle, minimum size = 1.17cm, fill = black, fill opacity = 0.2, text opacity = 1](t1+) at (10 + 1.5 + 0.0,-5.0) {$ \vartheta_{t+1}^1 $};

    \node[] at (00 + 2.0,-6.0 + 0.1) { \Huge$ \vdots $};
    \node[] at (05 + 2.0,-6.0 + 0.1) { \Huge$ \vdots $};
    \node[] at (10 + 2.0,-6.0 + 0.1) { \Huge$ \vdots $};

    \node[draw, circle, minimum size = 1.17cm, fill = black, fill opacity = 0.2, text opacity = 1](xH-) at (00 + 1.5 + 1.0,-2.0-5) {$ \vx_{t-1}^H $};
    \node[draw, circle, minimum size = 1.17cm, fill = black, fill opacity = 0.2, text opacity = 1](xH)  at (05 + 1.5 + 1.0,-2.0-5) {$ \vx_{t}^H $};
    \node[draw, circle, minimum size = 1.17cm, fill = black, fill opacity = 0.2, text opacity = 1](xH+) at (10 + 1.5 + 1.0,-2.0-5) {$ \vx_{t+1}^H $};
    \node[draw, circle, minimum size = 1.17cm, fill = black, fill opacity = 0.2, text opacity = 1](qH-) at (00 + 1.5 + 0.5,-3.5-5) {$ q_{t-1}^H $};
    \node[draw, circle, minimum size = 1.17cm, fill = black, fill opacity = 0.2, text opacity = 1](qH)  at (05 + 1.5 + 0.5,-3.5-5) {$ q_{t}^H $};
    \node[draw, circle, minimum size = 1.17cm, fill = black, fill opacity = 0.2, text opacity = 1](qH+) at (10 + 1.5 + 0.5,-3.5-5) {$ q_{t+1}^H $};
    \node[draw, circle, minimum size = 1.17cm, fill = black, fill opacity = 0.2, text opacity = 1](tH-) at (00 + 1.5 + 0.0,-5.0-5) {$ \vartheta_{t-1}^H $};
    \node[draw, circle, minimum size = 1.17cm, fill = black, fill opacity = 0.2, text opacity = 1](tH)  at (05 + 1.5 + 0.0,-5.0-5) {$ \vartheta_{t}^H $};
    \node[draw, circle, minimum size = 1.17cm, fill = black, fill opacity = 0.2, text opacity = 1](tH+) at (10 + 1.5 + 0.0,-5.0-5) {$ \vartheta_{t+1}^H $};

    \draw[thick, -{Latex}] (z-) -- (z);
    \draw[thick, -{Latex}] (z) -- (z+);
    \draw[thick, -{Latex}] (x1-) -- (x1);
    \draw[thick, -{Latex}] (x1)  -- (x1+);
    \draw[-{Latex}, rounded corners = 9pt] (z-.320) |- (x1-.160);
    \draw[-{Latex}, rounded corners = 9pt]  (z.320) |-  (x1.160);
    \draw[-{Latex}, rounded corners = 9pt] (z+.320) |- (x1+.160);
    \draw[-{Latex}, rounded corners = 9pt] (z-.300) |- (q1-.160);
    \draw[-{Latex}, rounded corners = 9pt]  (z.300) |-  (q1.160);
    \draw[-{Latex}, rounded corners = 9pt] (z+.300) |- (q1+.160);
    \draw[-{Latex}, rounded corners = 9pt] (z-.280) |- (t1-.160);
    \draw[-{Latex}, rounded corners = 9pt]  (z.280) |-  (t1.160);
    \draw[-{Latex}, rounded corners = 9pt] (z+.280) |- (t1+.160);

    \draw[-{Latex}, rounded corners = 9pt] (z-.260) |- (xH-.160);
    \draw[-{Latex}, rounded corners = 9pt]  (z.260) |-  (xH.160);
    \draw[-{Latex}, rounded corners = 9pt] (z+.260) |- (xH+.160);
    \draw[-{Latex}, rounded corners = 9pt] (z-.240) |- (qH-.160);
    \draw[-{Latex}, rounded corners = 9pt]  (z.240) |-  (qH.160);
    \draw[-{Latex}, rounded corners = 9pt] (z+.240) |- (qH+.160);
    \draw[-{Latex}, rounded corners = 9pt] (z-.220) |- (tH-.160);
    \draw[-{Latex}, rounded corners = 9pt]  (z.220) |-  (tH.160);
    \draw[-{Latex}, rounded corners = 9pt] (z+.220) |- (tH+.160);

    \draw[thick, dotted, -{Latex}] (z+)  --++ (3.5+0.25,0);
    \draw[thick, dotted, -{Latex}] (x1+) --++ (1.0+0.25,0);
    \draw[thick, dotted, -{Latex}] (q1+) --++ (1.5+0.25,0);
    \draw[thick, dotted, -{Latex}] (t1+) --++ (2.0+0.25,0);
    \draw[thick, dotted, -{Latex}] (xH+) --++ (1.0+0.25,0);
    \draw[thick, dotted, -{Latex}] (qH+) --++ (1.5+0.25,0);
    \draw[thick, dotted, -{Latex}] (tH+) --++ (2.0+0.25,0);

    \draw[thick, dotted, {Latex}-] (z-)  --++ (-1.5,0);
    \draw[thick, dotted, {Latex}-] (x1-) --++ (-4.0,0);
    \draw[thick, dotted, {Latex}-] (q1-) --++ (-3.5,0);
    \draw[thick, dotted, {Latex}-] (t1-) --++ (-3.0,0);
    \draw[thick, dotted, {Latex}-] (xH-) --++ (-4.0,0);
    \draw[thick, dotted, {Latex}-] (qH-) --++ (-3.5,0);
    \draw[thick, dotted, {Latex}-] (tH-) --++ (-3.0,0);

    \draw[thick, -{Latex}] (x1-) --(x1);
    \draw[thick, -{Latex}] (q1-) --(q1);
    \draw[thick, -{Latex}] (t1-) --(t1);
    \draw[thick, -{Latex}] (xH-) --(xH);
    \draw[thick, -{Latex}] (qH-) --(qH);
    \draw[thick, -{Latex}] (tH-) --(tH);

    \draw[thick, -{Latex}] (x1) --(x1+);
    \draw[thick, -{Latex}] (q1) --(q1+);
    \draw[thick, -{Latex}] (t1) --(t1+);
    \draw[thick, -{Latex}] (xH) --(xH+);
    \draw[thick, -{Latex}] (qH) --(qH+);
    \draw[thick, -{Latex}] (tH) --(tH+);

    \node[align=center] () at (-3.0, +0.0) {latent\\variable};
    \node () at (-3.0, -3.5) {1-st e.e.};
    \node () at (-3.0, -6.0) {$\Huge \vdots$};
    \node () at (-3.0, -8.5) {$H$-th e.e.};
    \draw [decorate, decoration={brace, mirror, amplitude=10pt}] (-1.65, -2) --++ (0, -3);
    \draw [decorate, decoration={brace, mirror, amplitude=10pt}] (-1.65, -2-5) --++ (0, -3);
\end{tikzpicture}
    }
    \caption{Graphical representation of a Pose+Gripper ARHMM.}
    \label{fig:pose_gripper_graph_mod}
\end{figure}
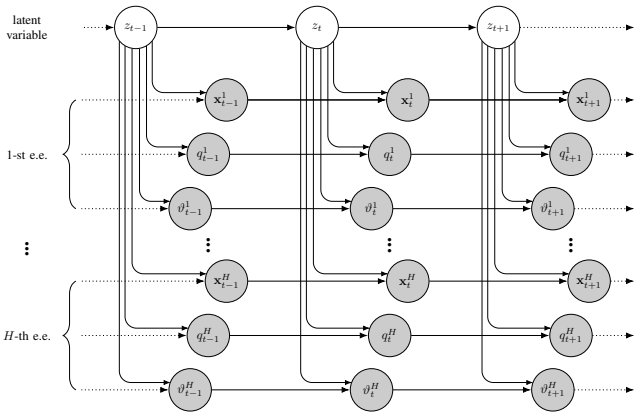

To demonstrate the improvement in the segmentation quality given by the different dynamics, we compare our model against a Linear ARHMM with the same topology, so that the only difference lies in the definition of the Auto Regressive dynamics itself.
This means that position, orientation, and gripper angle for each end effector is independent of each other, and the conditional probability of the next observed state follows the same law as in \eqref{eq:indep_pose_angle}.
The only difference is that the Linear ARHMM uses linear dynamics for the evolution of each observed state.
Since the conditional independence (and, thus, the topology) are identical, the graphical representation of the model does not change.

To test our method, we use the \emph{JIGSAW} dataset \cite{GVRAVLTZKH14}.
It consists of three different surgical tasks (`SUTuring', `Knot Tying', and `Needle Passing') executed by eight surgeons of different skill levels.
Data consists of positions, velocities, orientations (in the form of rotation matrices), angular velocities, and gripper angle.
Each variable is given for both left and right arms, and for both the patient-side and surgeon-side controllers.
Moreover, JIGSAW provides a segmentation of all the tasks in gestures.
We use only positions, orientations (converting from rotation matrices to unit quaternions) and gripper angle for the patient-side end-effectors.
Thus, in our case, $H=2$.

To evaluate the quality of the algorithm, we proceed as follows.
For a given batch of data, we use the Expectation Maximization algorithm to infer the set of parameters $ \Theta $ for both models.
Then, we apply the Viterbi algorithm to segment one of the executions that were not in the training set, and we compute different scores for segmentation.
This test is repeated a given number of times randomizing each time both the training and testing sets.

\noindent
The scores we decide to adopt are the following: \emph{Seg-score} and \emph{Silhouette Index} (SI).
Seg-score (also known as \emph{Jaccard index}) \cite{IH98, MUSMMEUEVB14} is a \emph{supervised score}, that is, it compares the obtained segmentation to the ground-truth;
it is the size of similarity between segmentation result and ground-truth.
SI \cite{Rou87}, instead, is an \emph{unsupervised score} and it evaluates the goodness of the segmentation by treating it as a clustering problem;
it evaluates the clusters by comparing the average distance within a cluster with the average distance to the points in the nearest clusters.
We decided to use these scores since they are already used for the evaluation of segmentation algorithms in the context of robotic surgery  \cite{FACE16, MGKPADG16}.

The main challenge in using this dataset is that it is not consistent: the same task may contains different gestures in each trial (that is, a particular gesture is present in one trial but not in the others).
This makes the segmentation a particularly hard problem: if the model has many hidden modes, but some are rare, over-fitting and over-segmentation (i.e. a segment is wrongly split in multiple smaller parts) will likely happen.
On the other hand, if we use fewer hidden modes, some different gestures will be described by the same hidden mode, resulting in a lower Seg-score.

\begin{figure}[t]
    \centering
    \subfloat[Knot Tying - 6 modes, Seg-score.]
        {\includegraphics[width = 0.45\linewidth]{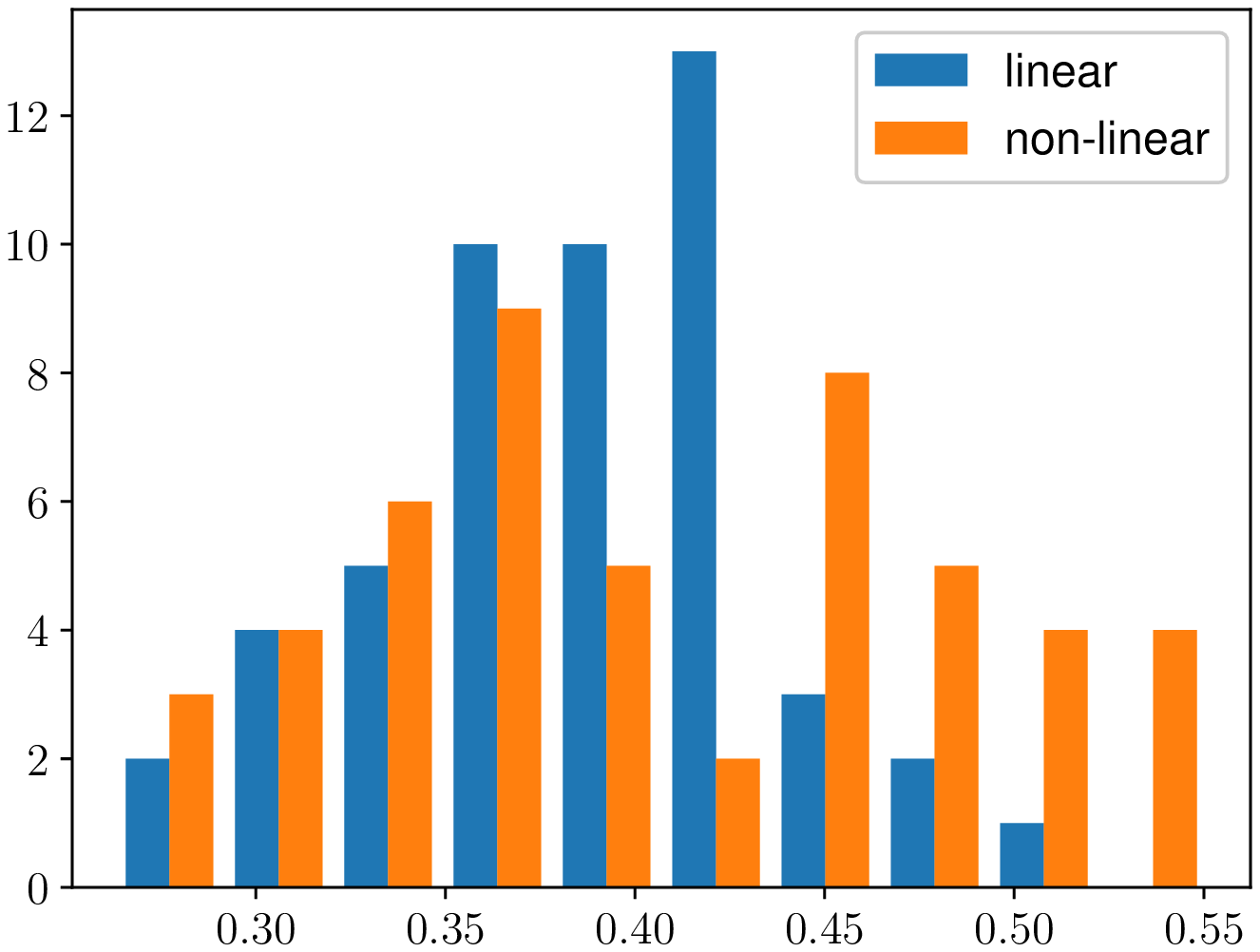}}
    \hfill
    \subfloat[Knot Tying - 6 modes, SI.]
        {\includegraphics[width = 0.45\linewidth]{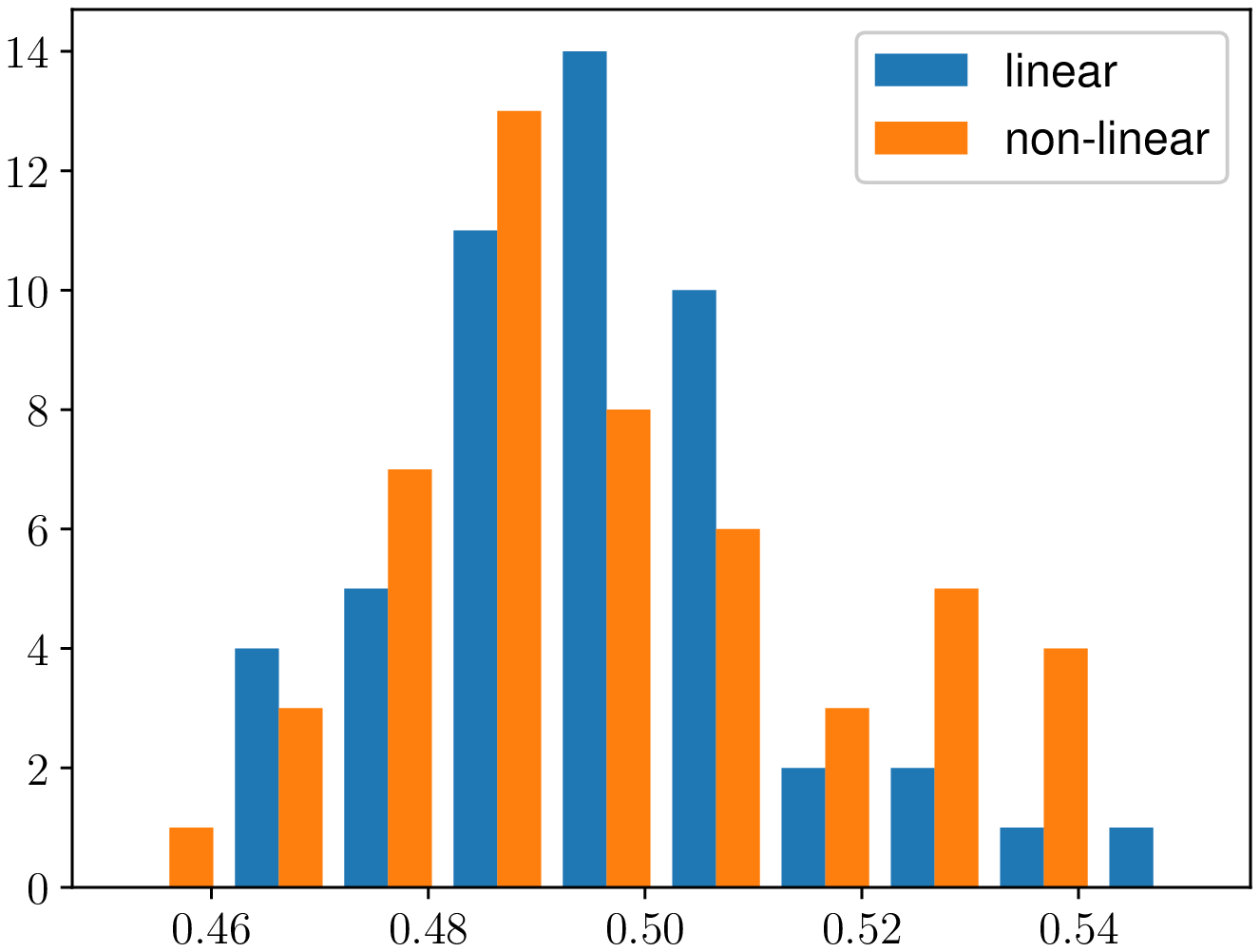}}
    \\
    \subfloat[Suturing - 5 modes, Seg-score.]
        {\includegraphics[width = 0.45\linewidth]{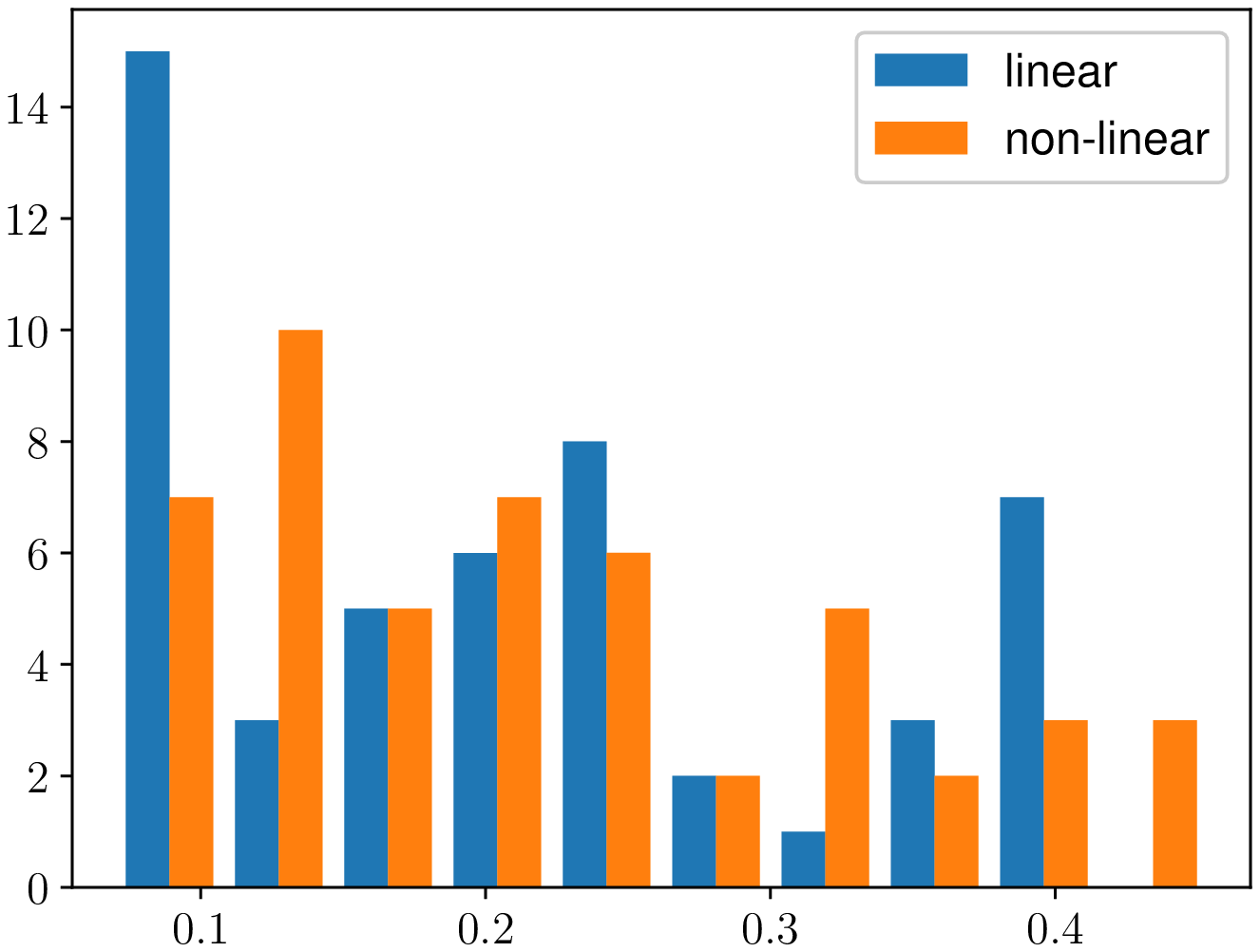}}
    \hfill
    \subfloat[Suturing - 5 modes, SI.\label{subfig:sut_5_si}]
        {\includegraphics[width = 0.45\linewidth]{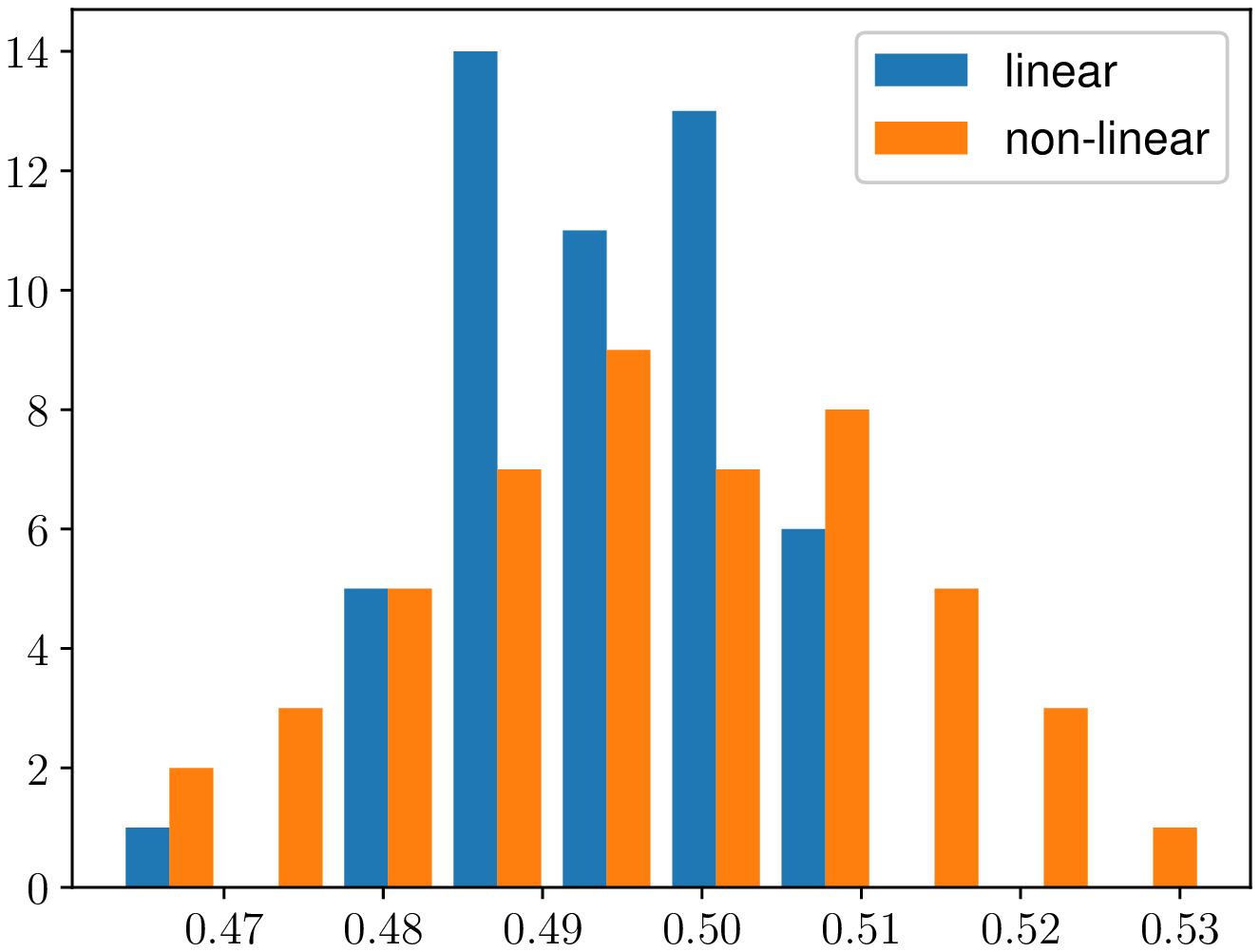}}
    \caption{
        Histograms showing the number of occurrences for different scores.
        Color blue marks the `state of the art' linear ARHMM, while orange marks our proposed generalization.
    }
    \label{fig:results_jigsaw}
\end{figure}

In Figure~\ref{fig:results_jigsaw} we show the results of this test.
In particular, for each of 100 trials with a given task and a given number of hidden modes, we show the histograms counting the number of occurrences of the scores in a given interval.
As it can be seen, the NL model is able to achieve, on average, a higher score in both the supervised and unsupervised metric, and it is also able to achieve higher scores (see, for instance, Figure~\ref{subfig:sut_5_si}).

\section{CONCLUSION AND FUTURE WORKS}
\label{sec:conclusion}


In this work, we proposed a generalization to the Auto-Regressive Hidden Markov Model via modifications of the Auto-Regressive dynamics.
In particular, we proposed a Non-Linear dynamics in Cartesian space and a linear dynamics in Unit-Quaternion space.

\noindent
Experiments on real datasets show that adopting these new dynamics result in an improvement in segmentation scores.
This has been proved by comparing two topologically identical models in which the only difference is the formulation of the vector fields governing the evolution of the observed state.

As future work, we aim to further generalize the observed state dynamics to dynamical systems used in trajectory learning for robotics such as \emph{Dynamic Movement Primitives} \cite{INS02, GSF21} and \emph{Probabilistic Movement Primitives} \cite{PDPN13}.
This would allow to simultaneously segment a robot trajectory and extract the robot movements used to generate the trajectory.

\noindent
Moreover, we aim at generalizing ARHMMs to deal with dynamics in generic Riemannian Manifolds and non-Euclidean spaces, possibly extending our proposed idea for Unit-Quaternions to the space $SO(3)$ of Rotation Matrices and the space of Symmetric Positive Definite (SPD) matrices, both heavily used in robotics.

\bibliographystyle{IEEEtran}
\bibliography{nl_arhmm_biblio.bib}

\end{document}